\documentclass[11pt,letterpaper]{article}
\usepackage[margin=1in]{geometry}

\usepackage[T1]{fontenc}
\usepackage[utf8]{inputenc}
\usepackage{lmodern}

\usepackage{graphicx}
\usepackage{xcolor}

\usepackage{amsmath,amssymb,amsfonts,amsthm}
\usepackage{mathtools}
\usepackage{bm}

\usepackage{listings}      

\usepackage[round]{natbib}

\usepackage[hyphens]{url}
\usepackage[colorlinks=true,linkcolor=black,citecolor=black,urlcolor=blue]{hyperref}
\usepackage{cleveref}

\newtheorem{theorem}{Theorem}
\newtheorem{lemma}{Lemma}

\newtheorem{corollary}{Corollary}

\newtheorem{assumption}{Assumption}

\theoremstyle{definition}

\theoremstyle{remark}
\newtheorem{remark}{Remark}

\crefname{theorem}{Theorem}{Theorems}
\Crefname{theorem}{Theorem}{Theorems}
\crefname{lemma}{Lemma}{Lemmas}
\Crefname{lemma}{Lemma}{Lemmas}
\crefname{corollary}{Corollary}{Corollaries}
\Crefname{corollary}{Corollary}{Corollaries}
\crefname{proposition}{Proposition}{Propositions}
\Crefname{proposition}{Proposition}{Propositions}
\crefname{definition}{Definition}{Definitions}
\Crefname{definition}{Definition}{Definitions}
\crefname{assumption}{Assumption}{Assumptions}
\Crefname{assumption}{Assumption}{Assumptions}
\crefname{claim}{Claim}{Claims}
\Crefname{claim}{Claim}{Claims}
\crefname{remark}{Remark}{Remarks}
\Crefname{remark}{Remark}{Remarks}
\crefname{figure}{Figure}{Figures}
\Crefname{figure}{Figure}{Figures}
\crefname{section}{Section}{Sections}
\Crefname{section}{Section}{Sections}

\newcommand{\bbP}{\mathbb{P}}
\newcommand{\vzero}{\mathbf{0}}
\newcommand{\bell}{\bm{\ell}}
\newcommand{\ind}{\mathbf{1}}
\newcommand{\ip}[2]{\left\langle #1, #2 \right \rangle}
\DeclarePairedDelimiter{\norm}{\lVert}{\rVert}
\DeclareMathOperator{\diag}{diag}
\DeclareMathOperator{\poly}{poly}
\DeclareMathOperator{\softmax}{softmax}
\DeclareMathOperator*{\argmax}{arg\,max}
\DeclareMathOperator*{\argmin}{arg\,min}

\AtBeginDocument{\let\cite\citep}

\newcommand{\stexnintlink}[1]{} 
\newcommand{\stexndeletedintlink}[1]{} 
\newcommand{\stexnextlink}[1]{} 
\newcommand{\stexndeletedextlink}[1]{} 

\newcommand{\stexntodo}[1]{} 
\newcommand{\stexnDELETEDtodo}[1]{} 

\newcommand{\bbR}{\mathbb{R}}

\newcommand{\bbE}{\mathbb{E}}

\newcommand{\bfe}{\mathbf{e}}

\newcommand{\bfh}{\mathbf{h}}
\newcommand{\bfH}{\mathbf{H}}

\newcommand{\bfI}{\mathbf{I}}

\newcommand{\bfu}{\mathbf{u}}
\newcommand{\bfU}{\mathbf{U}}

\newcommand{\bfW}{\mathbf{W}}
\newcommand{\bfx}{\mathbf{x}}

\newcommand{\calS}{\mathcal{S}}

\newcommand{\vones}{\mathbf{1}}

\newcommand{\coderepo}{https://github.com/skandaka/onenn-multiclass}

\title{One-Layer Transformer Provably Learns Multiclass\\
One-Nearest Neighbor in Context}

\author{%
  Skanda Athreya\thanks{\texttt{athreya1428@students.d211.org}} \\
  {\normalsize James B. Conant High School}\\
  {\normalsize Hoffman Estates, IL}
  \and
  Yutong Wang\thanks{\texttt{ywang562@illinoistech.edu}} \\
  {\normalsize Illinois Institute of Technology}\\
  {\normalsize Chicago, IL}
}

\date{}

\begin{document}

\maketitle

\begin{abstract}
We extend recent work establishing an equivalence between one-layer transformers and nearest-neighbor classifiers in the binary setting to the multiclass case. By leveraging the simplex encoding, we show that one-layer transformers with an argmax classification head behave identically to a one-nearest-neighbor classifier in the multiclass setting. This closes a gap left by prior work, whose multiclass result relied on a non-standard rounding-based approach rather than the typical argmax head used in practice.

\end{abstract}

\begin{center}
\textbf{Code:}\ \href{\coderepo}{\texttt{\coderepo}}
\end{center}

\section{Introduction}

Transformers have become the dominant architecture in large language models and other applications such as computer vision. Understanding the theoretical properties of transformers is therefore of considerable importance. A growing body of work studies theoretical properties of transformers, e.g., in the framework of in-context learning \cite{bai2023transformers},
and their connection to classical machine learning models such as nearest
neighbors \cite{li2024onelayer}.

\citet{li2024onelayer} showed that one-layer transformers, under appropriate conditions, behave identically to a one-nearest-neighbor classifier. Their work is, however, limited to the binary classification setting. While they do provide a result with respect to a task resembling multiclass classification \cite[Corollary 1]{li2024onelayer}, the result is for a rounding-based approach that departs from  the more standard argmax classification head used in practice.

In this work, we close this gap by leveraging the simplex encoding \cite{mroueh2012multiclass}.
We show that one-layer transformers with an argmax head are equivalent to one-nearest-neighbor classifiers in the multiclass setting.
This is a natural and important extension, since the primary application of transformers (\emph{i.e.,} next-token prediction) is inherently a multiclass classification problem, and the argmax head is the standard mechanism for producing predictions.


\section{Related works}
\subsection{Theory for transformers' generalization}
In-context learning (ICL) is the ability of a trained transformer to solve a new task from a  handful of input/output examples placed in its prompt, without any updates to its weights \cite{bai2023transformers}. Understanding how and why this behavior emerges from ordinary training is a central open problem in transformer theory, and a growing amount of work has aimed to investigate it in controlled settings.

One influential line of work casts ICL as implicitly running a learning algorithm on the in-context examples: e.g., \citet{zhang2023trained} proves that the behavior of a single-layer attention model trained on linear-regression prompts provably converges to a predictor with provable in-distribution and distribution shift guarantees.
In a similar vein, the one-nearest neighbor \cite{li2024onelayer} and our work
show that transformers can implicitly act like a one-nearest neighbor algorithm.

\subsection{Simplex encoding}
Representing $K$ classes as vectors is an important algorithmic design choice which decides how to convert a classifier's real-valued output to a discrete decision.
The simplex encoding assigns each class a vertex of a regular simplex, so that distinct classes are equidistant and symmetric.

\citet{mroueh2012multiclass} introduced this coding for multiclass learning and analyzed its geometric properties. Later, \citet{pouliot2018equivalence} established equivalences between simplex-based methods and other multi-category support vector machine formulations \cite{lee2004multicategory}.
The simplex encoding has also been applied fruitfully in representation learning \cite{papyan2020prevalence}.

\subsection{Nearest neighbor algorithms}
The nearest-neighbor algorithm is one of the oldest and most well-studied methods in classification.
\citet{cover1967nearest} established foundational results on its convergence properties, showing that the risk of the one-nearest-neighbor classifier is bounded above by twice the Bayes risk as the number of samples tends to infinity.
Subsequent work established consistency and finite-sample rates for the $k$-nearest-neighbor rule under margin and smoothness conditions: \citet{chaudhuri2014rates} gives distribution-dependent, finite-sample rates in metric spaces and shows that under the Tsybakov margin condition the nearest-neighbor rate matches known lower bounds for nonparametric classification, and \citet{efremenko2020fast} construct a Bayes-consistent classifier that keeps fast query time while attaining rates comparable to minimax.
The margin event that controls our distribution-shift result (\cref{sec:shift}) plays a role akin to classical margin conditions, where performance is controlled precisely when queries are unlikely to fall close to a boundary between classes.

A separate line of work connects nearest-neighbor and kernel methods to attention. \citet{tsai2019transformer} formulate attention as a kernel smoother over the inputs, with the kernel scores given by input similarities.
This perspective supports the equivalence between attention and the nearest-neighbor rule that we study.




Beyond its role in classical statistics, the nearest-neighbor rule has re-emerged as an important component of modern language models.
\citet{khandelwal2020generalization} introduced $k$NN-LM which is an ``interpolation'' of transformers and \(k\)-nearest neighbor algorithms.
\citet{xu2023why} analyzed why the retrieval-like component in $k$NN-LM helps generalization. These results indicate that nearest-neighbor 
is a fruitful avenue to understanding transformer-based language models.



\begin{figure}[t]
  \centering
  \includegraphics[width=\textwidth]{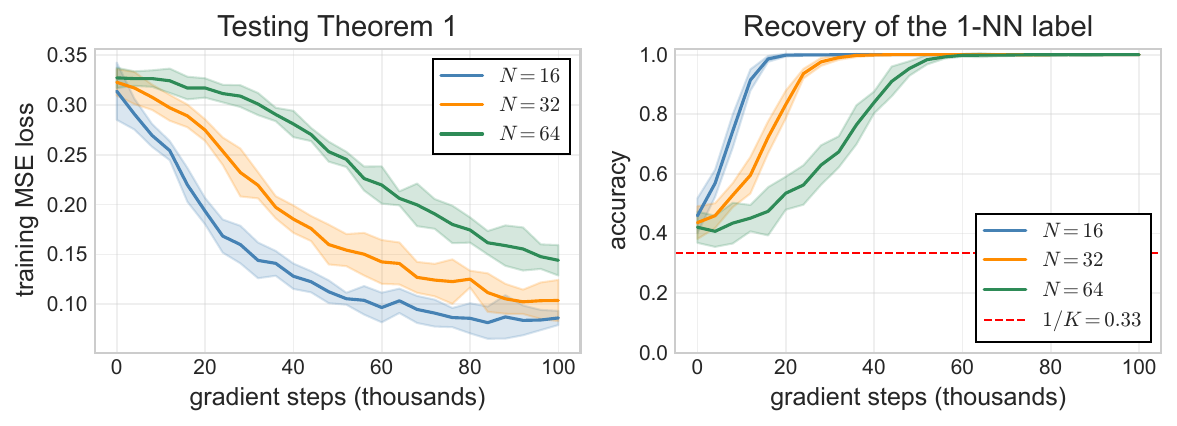}
  \caption{Convergence (\cref{thm:convergence}). Training loss (left) decreases
  and separated-test argmax accuracy (right) approaches $1$, faster for smaller
  context length $N$, matching the predicted dependence on $N$. Bands are
  $\pm 2$ standard deviations over $8$ seeds.
See
\Cref{sec:experiments} for experimental details.
}
  \label{fig:convergence}
\end{figure}

\begin{figure}[t]
  \centering
\includegraphics[width=\textwidth]{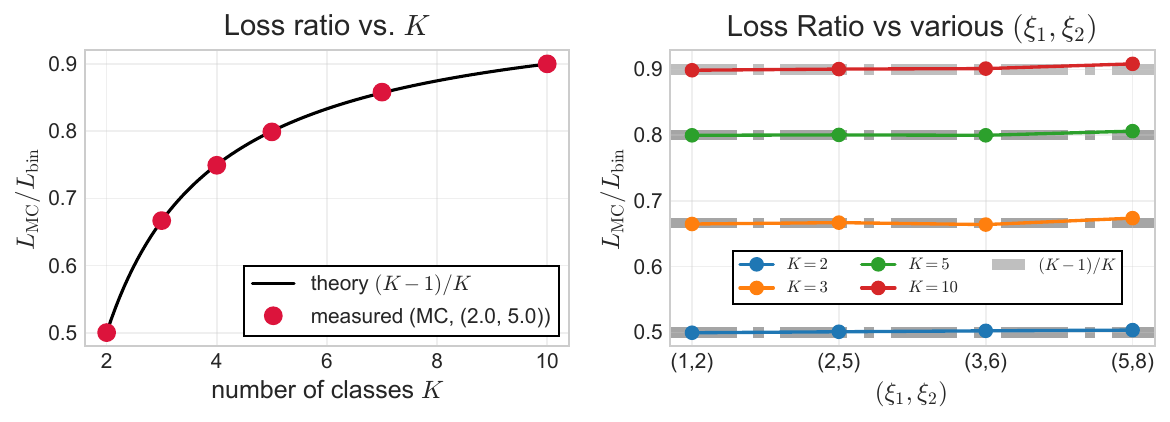}
  \caption{Scaling identity (\cref{lem:scaling}). The measured ratio
  $L_{\mathrm{MC}}/L_{\mathrm{bin}}$, estimated by sampling directly from the
  expected-version of the loss definitions, matches $(K-1)/K$ as a function of $K$ (left) and is
  constant in $(\xi_1,\xi_2)$ (right).}
  \label{fig:scaling}
\end{figure}

\section{Setup and assumptions}\label{sec:setup}
Throughout this work, let \(N\) be the number of labeled examples in a prompt
(the context length) and \(K\) the number of classes, both integers. Let
\([N] := \{ 1,\dots, N\}\).
Moreover, let \(d\) be the dimension of the input token.

\subsection{Problem Setup}
An \emph{in-context learning instance} consists of a prompt of \(N\) labeled examples
\(\{(\bfx_i, c_i)\}_{i \in [N]}\) and a query \(\bfx_{N+1}\), where each
\(\bfx_i \in \calS^{d-1}\) (the unit sphere in \(\bbR^d\)) and each class label
\(c_i \in [K]\).
The prompt
\(\{(\bfx_i, c_i)\}_{i \in [N]}\) and the query
\(\bfx_{N+1}\) are known to the learner, while the 
label \(c_{N+1}\) of the
query is hidden from the learner.
The objective is to predict \(c_{N+1}\).

Throughout we reserve \(y\) for the binary \(\{+1,-1\}\) label encoding used by
\citet{li2024onelayer} and write \(c\) for a multiclass label in \([K]\).
For the given prompt, we define the index of the nearest neighbor to the query by
\[i^* \coloneqq \argmin_{j \in [N]} \norm{\bfx_{N+1} - \bfx_j}_2.\]
Thus, the \emph{one-nearest-neighbor (1-NN) predictor} 
for the prompt
is \(c_{i^*}\).

\paragraph{Simplex encoding.}
To represent \(K\) classes as vectors, we use the simplex encoding
\cite{mroueh2012multiclass}, implemented as  centered one-hot vectors:
\begin{equation}\label{eq:simplex}
\bfu_c \coloneqq \bfe_c - \tfrac{1}{K}\vones_K \in \bbR^K, \qquad c \in [K],
\end{equation}
where \(\bfe_c\) is the \(c\)-th one-hot vector and
\(\vones_K\) is the all-ones vector. The vectors
\(\{\bfu_c\}_{c \in [K]}\) are the vertices of a regular simplex centered at the origin and satisfy:
\begin{itemize}
  \item (zero mean)
        \(\sum_{c \in [K]} \bfu_c = \vzero_K\),
  \item (constant norm)
        \(\norm{\bfu_c}_2^2 = \tfrac{K-1}{K}\) for each \(c \in [K]\), and
  \item (constant inner product)
        \(\ip{\bfu_c}{\bfu_{c'}} = -\tfrac{1}{K}\) for \(c \neq c'\).
\end{itemize}
        Consequently, the \(\bfu_c\)'s are
        equidistant with distance \[\norm{\bfu_c - \bfu_{c'}}_2 = \sqrt{2}.\]
        These three properties are the only structural facts about the
encoding that the convergence proof uses, and they act as the multiclass
counterparts of the \(\{+1,-1\}\) label encoding used in the binary case in
\cite{li2024onelayer}.

\paragraph{Prompt Embedding.}
Following \cite{li2024onelayer}, we embed the prompt and query as 
the following matrix 
\begin{equation}\label{eq:embedding}
\bfH =
\begin{pmatrix}
\bfx_1 & \cdots & \bfx_N & \bfx_{N+1}\\
\bfu_{c_1} & \cdots & \bfu_{c_N} & \vzero_K\\
0 & \cdots & 0 & 1
\end{pmatrix}
\in \bbR^{(d+K+1)\times(N+1)}
\end{equation}
Note that the second row of \(\bfH\), representing the labels, consists of columns that are \(K\)-dimensional vectors, in contrast to that of \cite[Eq.~(2.1)]{li2024onelayer}.
Denote the columns of \(\bfH\) by
\(\bfh_1, \dots, \bfh_{N+1}\), so that \(\bfh_j = [\bfx_j; \bfu_{c_j}; 0]\) for
\(j \in [N]\) and \(\bfh_{N+1} = [\bfx_{N+1}; \vzero_K; 1]\).
We note that the query's label is 
\(\vzero_K\) since its label is unknown to the learner.

\paragraph{Model.}
The one-layer softmax
attention model has a single weight matrix
\(\bfW \in \bbR^{(d+K+1)\times(d+K+1)}\) as parameter. The model
takes \(\bfH\) as input and outputs
\begin{equation}\label{eq:forward}
\bfH_{\bfW} := \bfH \cdot \softmax(\bfH^\top \bfW \bfH),
\end{equation}
where the softmax is applied column-wise.
Compared with \cite[Eq.~(2.2)]{li2024onelayer}, which writes the attention score as \(\bfH^\top \bfW_K^\top \bfW_Q \bfH\) with separate key and query matrices before merging them into a single matrix, we parametrize the model directly by the merged matrix \(\bfW\), playing the role of \(\bfW_K^\top \bfW_Q\): the model depends on the key and query matrices only through this product, and analyzing the merged matrix is the standard reduction in this line of work \cite{li2024onelayer, zhang2023trained}.

As in \cite{li2024onelayer}, the value matrix is frozen to the identity, so each output token is a convex combination of the input tokens.
The prediction is the \(K\)-dimensional
vector
obtained by taking the 
\((d+1)\)-st to \((d+K)\)-th rows of the last column of 
\(\bfH_{\bfW}\):
\begin{align}
\bell_{\bfW}(\bfx_{N+1}) &= [\bfH_{\bfW}]_{d+1:d+K,\, N+1} \nonumber\\
&= \sum_{j=1}^{N} q_j(\bfx, \bfW)\, \bfu_{c_j} \in \bbR^K,
\label{eq:logits}
\end{align}
where \(\bfx \coloneqq (\bfx_1, \dots, \bfx_{N+1})\) collects the prompt
inputs and the query, and the attention weights are
\begin{equation}\label{eq:attn}
q_j(\bfx,\bfW)
= \frac{\exp(\bfh_j^\top \bfW \bfh_{N+1})}
       {\sum_{l=1}^{N+1}\exp(\bfh_l^\top \bfW \bfh_{N+1})}.
\end{equation}
The predicted class is produced by the standard \emph{argmax head}
\begin{equation}\label{eq:argmax}
\hat{c}_{\bfW}(\bfx_{N+1}) \coloneqq \argmax_{k \in [K]}\,
[\bell_{\bfW}(\bfx_{N+1})]_k .
\end{equation}

\subsection{Assumptions on training distribution}
\label{sec:assumptions-on-training-distribution}
\begin{assumption}[Training distribution]\label{ass:train}
The inputs \(\{\bfx_i\}_{i \in [N+1]}\) are sampled i.i.d.\ from the uniform
distribution on \(\calS^{d-1}\), and the class labels \(c_i\) are sampled uniformly
from \([K]\), independently across tokens and of the inputs. We write
\(P^{\mathrm{train}}\) for the resulting joint law of the instance
\((\{(\bfx_i, c_i)\}_{i \in [N]}, \bfx_{N+1})\).

Note that the following hold:
\begin{enumerate}
  \item
\(\bbE[\bfu_{c_i} \mid \bfx_{1:N+1}] = \vzero_K\),
  \item
\(\bbE[\bfu_{c_i}\bfu_{c_j}^\top \mid \bfx_{1:N+1}] = \vzero_{K\times K}\) for
        \(i \neq j\), and
  \item
        the law of \(\bfx_{1:N+1}\) is invariant under
\(\bfx_{1:N+1} \mapsto -\bfx_{1:N+1}\).
\end{enumerate}
\end{assumption}

\paragraph{Expectations.} An expectation \(\bbE[\,\cdot\,]\) or probability
\(\bbP(\,\cdot\,)\) written without a subscript is always taken with respect to
the training distribution \(P^{\mathrm{train}}\) of \cref{ass:train}. Every other
law carries an explicit subscript, as in \eqref{eq:loss} and \cref{thm:shift}.

\begin{remark}\label{rem:train-dist}
\cref{ass:train} is the multiclass counterpart of the training distribution
of \cite[Assumption~1]{li2024onelayer}, adopted for the same reasons. Because
the labels are independent of the inputs, the prompt contains no parametric
relationship between \(\bfx\) and \(c\) that the model could exploit
instead. The only way to reduce the loss is to retrieve the label of the
nearest context example.

The task is nonetheless nontrivial: each prompt induces a geometrically complicated partition of the sphere. Moreover, the resulting objective
is nonconvex (we show this in \cref{lem:nonconvex}). Moreover, the analysis uses
\cref{ass:train} only through properties 1,2, and 3, which decouple the label
codes from the attention dynamics and force the off-diagonal gradient blocks
to vanish.
\end{remark}

\begin{assumption}[Initialization]\label{ass:init}
For a parameter \(\sigma > 0\), gradient descent is initialized at
\begin{equation*}
\bfW^0 = \begin{pmatrix}
\vzero_{(d+K)\times(d+K)} & \vzero_{d+K} \\
\vzero_{d+K}^\top & -\sigma
\end{pmatrix} \in \bbR^{(d+K+1)\times(d+K+1)},
\end{equation*}
that is, every entry vanishes except the final diagonal entry, which equals
\(-\sigma\) and suppresses the query's self-attention.
\end{assumption}

The initialization will be exploited as follows. First, the entry \(-\sigma\) plays
the role of the masking used in practical self-attention training, by
suppressing the query's attention to its own token, whose label slot is
\(\vzero_K\), so that attention concentrates on the labeled context
examples. \cite{li2024onelayer} motivate the same initialization this way,
with \(\sigma\) playing the role of a mask value that is effectively set to
infinity in practice.

Second, it is similar to
\citet{trockman2023mimetic}, who observe that in pretrained transformers the
merged key--query product has a dominant positive diagonal, approximately a
scaled identity, and show that imposing this structure at initialization
improves trainability.
Gradient descent started at \(\bfW^0\) remains on the
``diagonal family'' \(\diag\{\xi_1^t\bfI_d,\, \vzero_{K\times K},\, -\xi_2^t\}\) with
\(\xi_1^t\) increasing, as shown in the appendix, so the trained key--query
block remains a multiple of the identity throughout training.

\paragraph{Training.}
We train \(\bfW\) with gradient descent on the expected (population)
mean-squared error between the logits and the simplex code of the 1-NN
label,
\begin{equation}\label{eq:loss}
L(\bfW) = \tfrac{1}{2}\,\bbE_{P^{\mathrm{train}}}\big[\,
\norm{\bell_{\bfW}(\bfx_{N+1}) - \bfu_{c_{i^*}}}_2^2 \,\big],
\end{equation}
initialized according to \cref{ass:init}.


\section{Main results}
This section presents our main theoretical results. Omitted proofs are all included in the appendix.
The analysis reduces the dynamics to a two-parameter family of weight matrices,
which \cref{ass:init} initializes and which gradient descent provably preserves.

\begin{lemma}[Two-parameter reduction]\label{lem:2d}
Under \cref{ass:train,ass:init} there exist real sequences
\(\{\xi_1^t\}_{t\ge0}\) and \(\{\xi_2^t\}_{t\ge0}\) with
\(\bfW^t = \diag\{\xi_1^t\bfI_d,\, \vzero_{K\times K},\, -\xi_2^t\}\) for all
\(t \ge 0\). Moreover the blocks \(\bfW_{12}, \bfW_{22}, \bfW_{32}\), which
multiply the query's zero label block, have identically vanishing gradient,
so the loss does not depend on them at any point of the parameter space.
\end{lemma}

Write \(L_{\mathrm{MC}}(\xi_1,\xi_2) \coloneqq L(\bfW)\) for the loss
\eqref{eq:loss} restricted to this 2-parameter space.
Next, we briefly recall the binary case from \cite{li2024onelayer}.
Let
\(L_{\mathrm{bin}}(\xi_1,\xi_2)\) be the binary objective from
\cite[Eq.~2.5]{li2024onelayer}, which we recall below.

For the binary case, the prompt and query are the same as in \cref{ass:train}. The
simplex codes \(\bfu_{c_j}\) are however replaced by scalar labels
\(y_j \in \{\pm1\}\) drawn uniformly and independently of the inputs, the logit
\(\ell_{\bfW}(\bfx_{N+1}) = \sum_{j=1}^{N} q_j(\bfx,\bfW)\, y_j\) is a scalar,
and the target is \(y_{i^*}\), so that
\begin{equation}\label{eq:lbin}
L_{\mathrm{bin}}(\xi_1,\xi_2)
= \tfrac{1}{2}\,\bbE\big[\big(\ell_{\bfW}(\bfx_{N+1}) - y_{i^*}\big)^2\big].
\end{equation}
The attention weights \(q_j\) are identical in the two models, because by
\cref{lem:2d} they depend on the inputs only through their inner products and
not on any label.

Now, the two losses are related by the identity:

\begin{lemma}[Scaling identity]\label{lem:scaling}
On the two-parameter family,
\(L_{\mathrm{MC}}(\xi_1,\xi_2) = \tfrac{K-1}{K}\,L_{\mathrm{bin}}(\xi_1,\xi_2)\).
Consequently \(\partial_{\xi_m} L_{\mathrm{MC}} = \tfrac{K-1}{K}\,
\partial_{\xi_m} L_{\mathrm{bin}}\) for \(m \in \{1,2\}\).
\end{lemma}

Every increment bound in the convergence analysis is therefore the binary
bound scaled by \(\tfrac{K-1}{K}\). The factor rescales the growth of
\((\xi_1, \xi_2)\) but cancels in the ratio that controls the trajectory.

\begin{theorem}[Convergence]\label{thm:convergence}
Suppose \cref{ass:train,ass:init} hold with \(N \ge O(\sqrt{d}\log d)\) and
\begin{align*}
\sigma > 2\max\Big\{\log(Nd),\;&
-\log\big(1-(N\sqrt{d})^{-1/d}\big),\;
                                 \\&
C_d\big(1-\tfrac{1}{2N}\big)\Big\},
\end{align*}
where \(C_d = \poly(d)\). These are the conditions of
\cite[Theorem~1]{li2024onelayer}, which we inherit unchanged.\footnote{%
\cite[Theorem~1]{li2024onelayer} print the middle term with exponent
\(1/d\). Since \((N\sqrt d)^{1/d} > 1\), the logarithm is then applied to a
negative number. We state the well-posed form with exponent \(-1/d\), which is
the quantity their Lemma~12 identifies this condition with, namely
\(-\log a_{N,d}\) up to constants.}
Then gradient descent on \eqref{eq:loss} satisfies
\begin{equation*}
L(\bfW^t) \le O\!\left(\frac{K-1}{K}\cdot
\frac{\poly(N,d)}{\log t}\right),
\end{equation*}
so \(L(\bfW^t)\) converges to \(0\) as \(t \to \infty\).
\end{theorem}

\subsection{Behavior Under Distribution Shift and the Argmax Head}
\label{sec:shift}
For \(\delta > 0\) define the event
\begin{multline}\label{eq:margin}
A_\delta \coloneqq \big\{\, \norm{\bfx_j - \bfx_{N+1}}_2^2 \ge 
\norm{\bfx_{i^*} - \bfx_{N+1}}_2^2 + \delta, \\
\text{for all } j \neq i^* \text{ with } c_j \neq c_{i^*} \,\big\}.
\end{multline}
which captures the set of instances with large margin.

\begin{theorem}[Distribution shift]\label{thm:shift}
Let \(P^{\mathrm{test}}\) be any distribution over instances
\((\{(\bfx_i, c_i)\}_{i \in [N]}, \bfx_{N+1})\) supported on \(\calS^{d-1}\)
with labels in \([K]\). After \(T\) steps of gradient descent,
\begin{multline*}
\bbE_{P^{\mathrm{test}}}\big[\norm{\bell_{\bfW^{T}}(\bfx_{N+1})
- \bfu_{c_{i^*}}}_2^2\big] \\
\le O\!\left(\tfrac{K-1}{K}\,\inf_{\delta > 0}\big\{
N^2 T^{-\poly(N,d)\,\delta}
+ P^{\mathrm{test}}(A_\delta^c) \big\}\right).
\end{multline*}
\end{theorem}
\begin{remark}[Training vs.\ test distributions]\label{rem:shift}
Here \(P^{\mathrm{test}}\) is a distribution over the entire instance
\((\{(\bfx_i,c_i)\}_{i\in[N]}, \bfx_{N+1})\). Beyond the stated ``supported on \(\calS^{d-1}\)''
conditions, no independence is assumed between inputs and
labels, or across tokens, and no relation to \(P^{\mathrm{train}}\) is
assumed.
This mirrors \cite[Assumption~3]{li2024onelayer}. The label
boundedness required there holds automatically here since
\(\norm{\bfu_c}_2 \le 1\).

A guarantee under an arbitrary
\(P^{\mathrm{test}}\) is possible because the underlying estimate is
pointwise rather than distributional: training on \(P^{\mathrm{train}}\)
produces the fixed parameters
\(\bfW^{T} = \diag\{\xi_1^{T}\bfI_d,\, \vzero_{K\times K},\, -\xi_2^{T}\}\),
and on the event \(A_\delta\) the deviation
\(\norm{\bell_{\bfW^{T}}(\bfx_{N+1}) - \bfu_{c_{i^*}}}_2\) admits a
deterministic bound, decaying in \(\xi_1^{T}\delta\) and \(\xi_2^{T}\),
which holds for every prompt--query configuration on the sphere (the
pointwise bound is derived in the appendix). The test distribution enters only through
\(P^{\mathrm{test}}(A_\delta^c)\), the probability that the query falls
within \(\delta\) of a 1-NN decision boundary.
\end{remark}
\begin{corollary}[Argmax classification]\label{cor:argmax}
Under the setting of \cref{thm:shift},
\begin{multline*}
P^{\mathrm{test}}\big(\hat{c}_{\bfW^{T}}(\bfx_{N+1}) \neq c_{i^*}\big) \\
\le O\!\left(\inf_{\delta > 0}\big\{
N^2 T^{-\poly(N,d)\,\delta} + P^{\mathrm{test}}(A_\delta^c) \big\}\right).
\end{multline*}
Moreover, if \(P^{\mathrm{test}}(A_{\delta^*}) = 1\) for some \(\delta^* > 0\),
then \(\hat{c}_{\bfW^{T}}(\bfx_{N+1}) = c_{i^*}\) almost surely once
\(\log T \ge O\!\big(\log(KN)/(\poly(N,d)\,\delta^*)\big)\).
\end{corollary}


\begin{figure}[t]
  \centering
  \includegraphics[width=\textwidth]{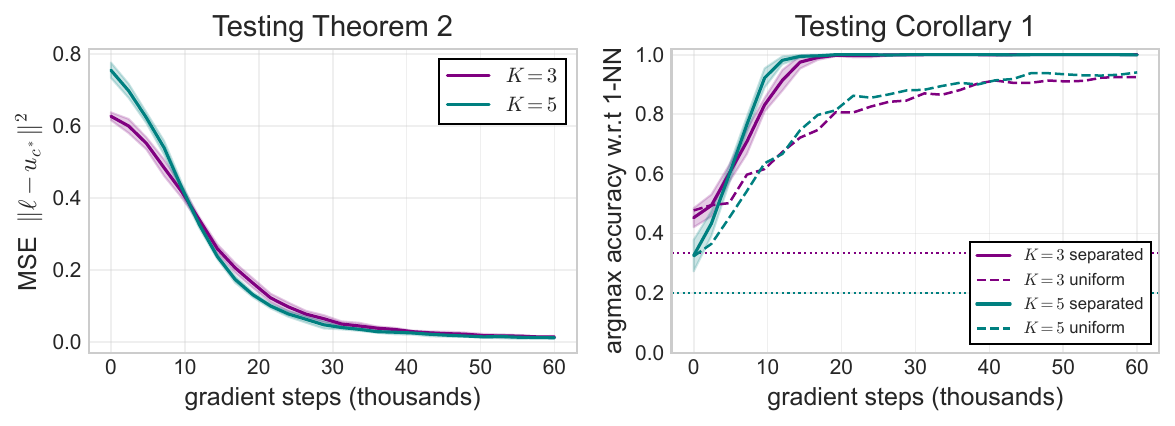}
  \caption{Distribution shift (\cref{thm:shift}) and argmax classification
  (\cref{cor:argmax}). On the separated test distribution the logit MSE
  $\norm{\bell - \bfu_{c_{i^*}}}_2^2 \to 0$ (left) and the argmax matches the
  1-NN label (right). The uniform-data accuracy is shown for contrast.}
  \label{fig:distshift}
\end{figure}

\begin{figure}[t]
  \centering
  \includegraphics[width=\textwidth]{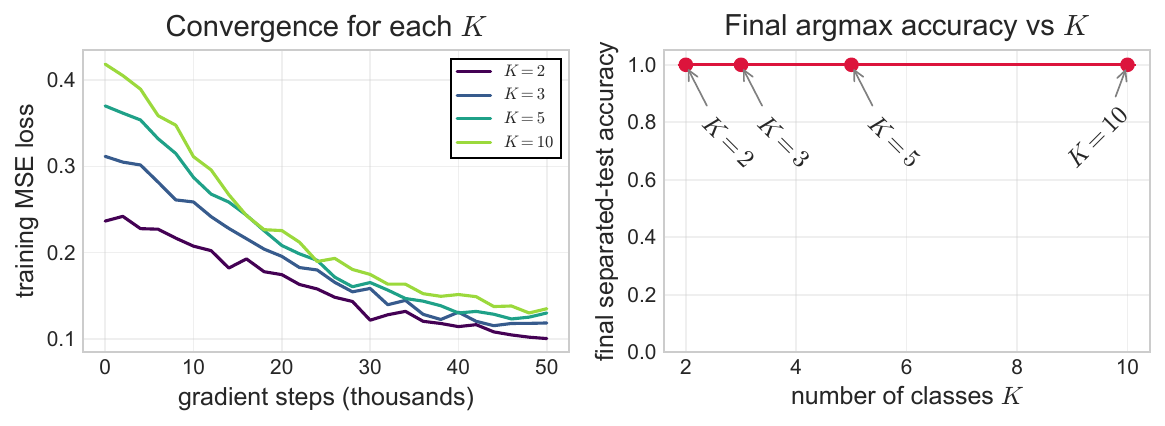}
  \caption{Convergence holds for every $K \in \{2,3,5,10\}$ (left), with final
  separated-test accuracy $1$ across $K$ (right).}
  \label{fig:kdep}
\end{figure}

\begin{figure}[t]
  \centering
  \includegraphics[width=\textwidth]{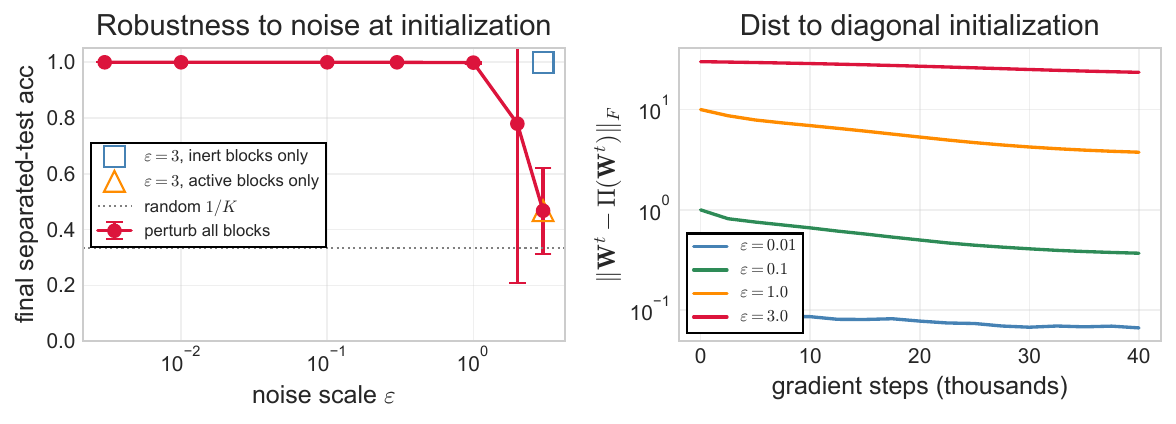}
  \caption{Perturbed initialization. Left: final separated-test argmax accuracy
  against the noise scale $\varepsilon$, with the two block-restricted
  ablations at $\varepsilon = 3$. Right: distance from the two-parameter family
  of \cref{lem:2d} over training, for the blocks that enter the attention
  scores.}
  \label{fig:init}
\end{figure}

\begin{figure}[t]
  \centering
  \includegraphics[width=\textwidth]{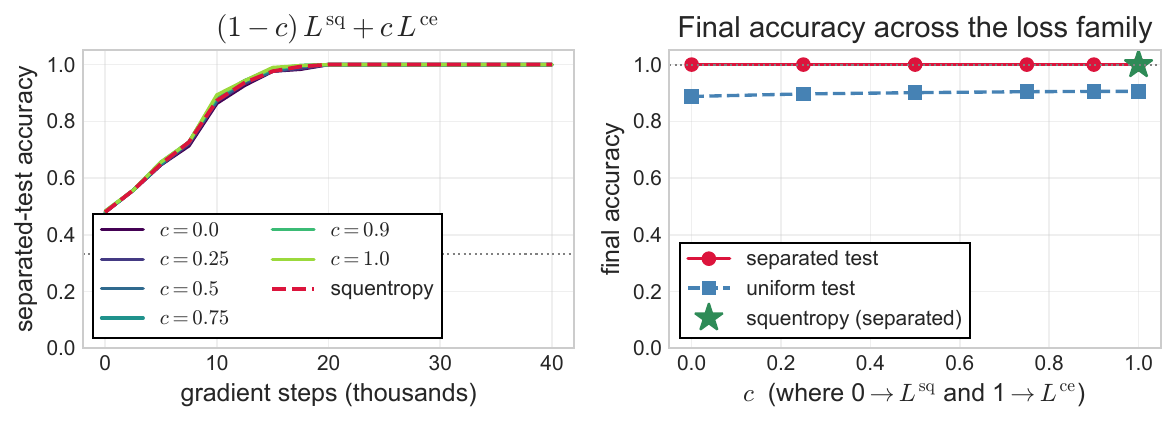}
  \caption{Loss function. Left: separated-test argmax accuracy over training for
  $(1-c)L^{\mathrm{sq}} + cL^{\mathrm{ce}}$ and for squentropy. Right: final
  accuracy on the separated and uniform test distributions across the family.}
  \label{fig:loss}
\end{figure}

\section{Experiments}
\label{sec:experiments}

We validate the theory with six experiments. Complete configurations and code to reproduce all experimental results
are available at \href{\coderepo}{\texttt{\coderepo}}.
Compute environment is listed in \Cref{app:hardware}.

\paragraph{Protocol.}
All experiments use the one-layer attention model of \eqref{eq:forward} with
the score $\bfH^\top \bfW \bfH$, the centered one-hot vector encoding
\eqref{eq:simplex}, and the loss \eqref{eq:loss}. Training data follows
\cref{ass:train} (inputs i.i.d.\ uniform on $\calS^{d-1}$, labels uniform on
$[K]$), and $\bfW$ is initialized as in \cref{ass:init}. We run stochastic
gradient descent on minibatch estimates of \eqref{eq:loss} with analytical
gradients, which we verify against finite differences in double precision to relative
error $\approx 10^{-10}$. Training runs in PyTorch \cite{paszke2019pytorch}.

The theory analyzes
gradient descent on the population objective \eqref{eq:loss}. The minibatch
updates are a stochastic approximation of the population gradient.
\cite{li2024onelayer} likewise validate their theory with stochastic gradient
descent, in a less restrictive setting still with random initialization
 whereas we retain the initialization of \cref{ass:init}. Unless stated
otherwise, $d = 8$, $N = 16$, $\sigma = 3.0$, learning rate $\eta = 0.05$, and
batch size $128$.

Additionally,
 we evaluate on a
well-\emph{separated} test distribution, extending the construction of
\cite[Appendix~B]{li2024onelayer} to $K$ classes,  which places
the query strictly closer to examples of a single class than to all others, so
that $P^{\mathrm{test}}(A_{\delta^*}) = 1$ for a fixed margin $\delta^* > 0$,
matching the setting of \cref{thm:shift} and \cref{cor:argmax}.

\paragraph{Scaling identity (\cref{fig:scaling}).}
We test \cref{lem:scaling} directly. The identity
is an exact identity between two population expectations that admit no closed
form, so we verify it by estimating both sides directly from their definitions.
For fixed $(\xi_1,\xi_2)$, we sample prompts and queries from \cref{ass:train},
estimate $L_{\mathrm{MC}}$ and $L_{\mathrm{bin}}$ on this sample, and plot their
ratio.

At $(\xi_1,\xi_2) = (2.0, 5.0)$ the measured ratio traces $(K-1)/K$ over
$K \in \{2,3,4,5,7,10\}$ ($80000$ samples per point, left), and it is constant
over the grid $(\xi_1,\xi_2) \in \{(1,2),\,(2,5),\,(3,6),\,(5,8)\}$ for
$K \in \{2,3,5,10\}$ ($40000$ samples per point, right), with deviations within
error tolerance
confirming
\cref{lem:scaling}.

\paragraph{Convergence (\cref{fig:convergence}).}
We train with $K = 3$ and context lengths $N \in \{16, 32, 64\}$ for $10^5$
steps and $8$ random seeds. The training loss decreases towards zero at the
slow logarithmic rate of \cref{thm:convergence}, more slowly for larger $N$,
consistent with the $\poly(N,d)$ dependence of the bound (left). Argmax accuracy
on the separated test distribution rises from the random-guess baseline
$1/K \approx 0.33$ to $1$ (right). Bands show $\pm 2$ standard deviations over
the seeds.

\paragraph{Distribution shift and the argmax head (\cref{fig:distshift}).}
For $K \in \{3, 5\}$ we train for $6 \times 10^4$ steps and $6$ random seeds,
and evaluate on the separated distribution, which differs from the training
distribution but satisfies the margin condition. The logit error
$\norm{\bell - \bfu_{c_{i^*}}}_2^2$ decays toward zero (left) and argmax
accuracy reaches $1$ (right), as \cref{thm:shift} and \cref{cor:argmax}
predict.

On uniform test data, on the other hand, accuracy plateaus below $1$:
a constant fraction of queries falls near a 1-NN decision boundary, so
$P^{\mathrm{test}}(A_\delta^c)$ does not vanish for any fixed $\delta$ and the
second term of the bound in \cref{thm:shift} remains. We note that the plateau is as predicted by the
theory.

\paragraph{Dependence on the number of classes (\cref{fig:kdep}).}
We vary $K \in \{2,3,5,10\}$, training for $5 \times 10^4$ steps and $4$
random seeds. Convergence holds for every $K$: the prefactor $\tfrac{K-1}{K}$
shifts the loss curves but leaves the logarithmic rate unchanged (left), and the
final separated-test accuracy is $1$ for every $K$ (right). The guarantees of
\cref{thm:convergence} and \cref{cor:argmax} are therefore uniform in the number
of classes.

\paragraph{Dependence on the diagonal initialization (\cref{fig:init}).}
\cref{ass:init} prescribes an exactly diagonal $\bfW^0$, which has precedent in
the mimetic initialization of \cite{trockman2023mimetic}. We test how far the
1-NN phenomenon depends on it by adding i.i.d.\ $\mathcal{N}(0,\varepsilon^2)$
noise to every entry of $\bfW^0$ and sweeping $\varepsilon$, training for
$4 \times 10^4$ steps and $3$ random seeds at each value.

Alongside accuracy we
track the distance from the two-parameter family of \cref{lem:2d}, that is
$\norm{\bfW^t - \Pi(\bfW^t)}_F$ where $\Pi$ projects onto
$\diag\{\xi_1\bfI_d, \vzero_{K\times K}, -\xi_2\}$ (\cref{fig:init} right).

The separated-test argmax accuracy is $1.000$ for every $\varepsilon \le 1.0$,
degrades at $\varepsilon = 2.0$ ($0.78 \pm 0.28$) and fails at
$\varepsilon = 3.0$ ($0.44 \pm 0.07$), so the guarantee survives perturbations
up to roughly the scale $\sigma$ of the mask entry itself.

We observe first that, wherever the result survives, the distance
to the family \emph{contracts} over training, to about a third of its initial
value ($1.056 \to 0.343$ at $\varepsilon = 0.1$ and $10.558 \to 3.449$ at
$\varepsilon = 1.0$); where it fails, the contraction weakens to $0.79$
($31.673 \to 25.145$ at $\varepsilon = 3.0$).
\cref{lem:2d} states that the
family is invariant under gradient descent, and this indicates that it is also
locally attracting: the trajectory returns to the diagonal structure rather
than merely remaining on it. Moreover, failing to return to the diagonal structure seems to imply failing to
learn the 1-NN rule.

Second, we observe that the component of the perturbation lying in
$\bfW_{12}, \bfW_{22}, \bfW_{32}$ is unchanged to three decimal places
throughout training ($5.400 \to 5.400$ at $\varepsilon = 1.0$), as the
identically zero gradients established in \cref{lem:2d} require.
Consistently, noise of magnitude $\varepsilon = 3.0$ confined to those blocks
leaves the accuracy at $1.000$, while the same magnitude confined to the
remaining blocks reproduces the failure ($0.44$).

\paragraph{Dependence on the loss function (\cref{fig:loss}).}
The analysis thus far (and also in \cite{li2024onelayer}) is specific to the square loss: \cref{lem:scaling} is an identity
between coordinatewise expansions of \eqref{eq:loss} and has no cross-entropy
analogue.
We test whether the one-nearest-neighbor phenomenon
extends to ``convex-combination'' losses of the form
\[(1-c)\,L^{\mathrm{sq}} + c\,L^{\mathrm{ce}},\] interpolating from the square
loss at $c = 0$ to cross entropy at $c = 1$.
We also test against 
the ``squentropy'' loss
of \cite{hui2023cut},
\begin{equation*}
L^{\mathrm{sqen}} = L^{\mathrm{ce}}
+ \frac{1}{K-1}\sum_{k \neq c_{i^*}} [\bell]_k^2,
\end{equation*}
Unlike the convex-combination losses above, the squentropy loss 
has no tunable parameter.
Both
families are trained for $4 \times 10^4$ steps and $3$ random seeds, using the
same hyperparameters as in the square-loss experiments.

The separated-test argmax accuracy is $1.000$ for every $c > 0$ and for
squentropy, and $0.999 \pm 0.002$ at $c = 0$, so the recovery of the 1-NN label
does not depend on the choice of loss.
On uniform test data the accuracy
increases mildly and monotonically with $c$, from $0.908$ at $c = 0$ to $0.924$
at $c = 1$, with squentropy at $0.916$.

Recall that the square-loss gradient \(\approx 0\) when $\bell \approx \bfu_{c_{i^*}}$, whereas the
cross-entropy gradient never vanishes.
Training under cross entropy
therefore continues to boost attention on the nearest neighbor, reducing the
near-boundary errors seen in the uniform-data plateau.

Taken together, our result shows that the one-nearest-neighbor phenomenon is robust to changing the loss function. While our theory covers the square loss only, an interesting line of future work will analyze the cross entropy.


\section{Discussion, Limitations, and Future Work}\label{sec:discussion}
We have rigorously shown the equivalence between a one-layer softmax attention model and the one-nearest-neighbor rule, extending the binary setting result by \citet{li2024onelayer}.
We show that the convergence rate, the conditions on $\sigma$ and $N$, and the distribution-shift guarantee all carry over without $K$-dependence.

We discuss several limitations in and future directions of this work. The analysis exploits the one-layer architecture and the special diagonal initialization of \cref{ass:init}. While these restrictions have precedent in the literature
(e.g., in \citep{trockman2023mimetic})
an interesting avenue of future research is to remove these restrictions.

Our experiments show the one-nearest-neighbor phenomenon persists after switching the loss function from square loss to squentropy. It will be interesting to
develop a rigorous theory to explain this phenomenon.
The scaling identity is specific to the square loss and has no cross-entropy analogue, though the empirical behavior does not depend on that choice.

Finally, the guarantee concerns recovery of the 1-NN rule for a single head.
Extending the argument to multiple heads is a natural next step.


\bibliographystyle{plainnat}
\bibliography{references_YW}

\appendix
\section{Proof of the Convergence Theorem}\label{app:convergence}

We prove \cref{thm:convergence} as follows.

\begin{lemma}[Properties of the simplex code]\label{lem:simplex-props}
Under \cref{ass:train}, the simplex codes $\bfu_{c_i}$ defined in \eqref{eq:simplex} satisfy
\begin{gather*}
\text{(i)}\ \ \bbE[\bfu_{c_i} \mid \bfx_{1:N+1}] = \vzero_K,\\
\text{(ii)}\ \ \bbE[\bfu_{c_i}\bfu_{c_j}^\top \mid \bfx_{1:N+1}]
= \vzero_{K\times K}\ \text{ for } i \neq j,\\
\text{(iii)}\ \ \norm{\bfu_{c_i}}_2^2 = \tfrac{K-1}{K}.
\end{gather*}
\end{lemma}

\begin{proof}
For (i), the labels are independent of the inputs and the $K$ classes are
equally likely, so
$\bbE[\bfu_{c_i} \mid \bfx_{1:N+1}] = \bbE[\bfu_{c_i}]
= \tfrac{1}{K}\sum_{c=1}^{K}(\bfe_c - \tfrac1K\vones_K)
= \tfrac1K\vones_K - \tfrac1K\vones_K = \vzero_K$. For (ii), $c_i$ and $c_j$ are
independent for $i \neq j$ and both are independent of $\bfx_{1:N+1}$, so
$\bbE[\bfu_{c_i}\bfu_{c_j}^\top \mid \bfx_{1:N+1}]
= \bbE[\bfu_{c_i}]\,\bbE[\bfu_{c_j}]^\top = \vzero_{K\times K}$ by (i). For (iii),
for any class $c$,
\begin{align*}
\norm{\bfu_c}_2^2 &= \norm{\bfe_c - \tfrac1K\vones_K}_2^2\\
&= \norm{\bfe_c}_2^2 - \tfrac{2}{K}\ip{\bfe_c}{\vones_K}
+ \tfrac{1}{K^2}\norm{\vones_K}_2^2\\
&= 1 - \tfrac2K + \tfrac1K = \tfrac{K-1}{K},
\end{align*}
which is independent of $c$.
\end{proof}

\subsection{Closed-Form Gradient}

\begin{lemma}[Closed-form gradient]\label{lem:closed-form-grad}
Under \cref{ass:train,ass:init}, for all $t \ge 0$,
\begin{equation}\label{eq:grad-W11}
\begin{split}
\nabla_{\bfW_{11}} L(\bfW^t) = \bbE\!\bigg[
&\sum_{i=1}^{N} g_i^t\!\big(\ip{\bfx_i}{\bfx_{N+1}}\big)\,
\bfx_i \bfx_{N+1}^\top \\
&+ g_{i^*}^t\!\big(\ip{\bfx_{i^*}}{\bfx_{N+1}}\big)\,
\bfx_{i^*} \bfx_{N+1}^\top\bigg],
\end{split}
\end{equation}
for scalar-valued functions $\{g_i^t\}_{i \in [N]} \cup \{g_{i^*}^t\}$ of one real
variable, and $\nabla_{\bfW_{ij}} L(\bfW^t) = \vzero$ for every block $(i,j)$
except $(1,1)$ and $(3,3)$.
\end{lemma}

\begin{proof}
We induct on $t$. The form holds at $t=0$ by \cref{ass:init}. Assume it holds at step $t \ge 0$, so that
$\bfW^t = \diag\{\xi_1^t\bfI_d, \vzero_{K\times K}, -\xi_2^t\}$. Using
$\norm{\bfx_{N+1}}_2 = 1$, the score $\bfh_j^\top \bfW^t \bfh_{N+1}$ of
\eqref{eq:forward} reduces to $\xi_1^t \ip{\bfx_j}{\bfx_{N+1}}$ for
$j \in [N]$ and to $\xi_1^t - \xi_2^t$ for the query, so the attention
weights take the scalar form
\begin{equation}\label{eq:attn-scalar}
q_j = \frac{\exp\big(\xi_1^t \ip{\bfx_j}{\bfx_{N+1}}\big)}
{\sum_{l=1}^{N}\exp\big(\xi_1^t \ip{\bfx_l}{\bfx_{N+1}}\big)
 + \exp(\xi_1^t-\xi_2^t)}
\end{equation}
for $j \in [N]$, with $\exp(\xi_1^t-\xi_2^t)$ in the numerator for
$j = N+1$. Below we abbreviate $q_j = q_j(\bfx,\bfW^t)$. This is the
multiclass counterpart of \cite[Eq.~(C.9)]{li2024onelayer}.

The consequence used repeatedly below is that no label code appears in
\eqref{eq:attn-scalar}: each $q_j$ is a measurable function of
$\bfx_{1:N+1}$ alone. Hence for any function $\Phi$ of the labels and any
$j, j' \in [N]$,
\begin{equation}\label{eq:tower}
\bbE\big[\,q_j q_{j'}\, \Phi(c_{1:N})\,\big]
= \bbE\big[\,q_j q_{j'}\;
\bbE[\,\Phi(c_{1:N}) \mid \bfx_{1:N+1}\,]\,\big],
\end{equation}
by the tower property: the weights may be taken outside the inner
expectation, which is then evaluated by \cref{lem:simplex-props}. This step is
the only place where the independence of the labels from the inputs in
\cref{ass:train} is used.

Partition \(\bfW\) conformally with \(\bfh_j = [\bfx_j; \bfu_{c_j}; z_j]\),
\begin{equation*}
\bfW = \begin{bmatrix}
\bfW_{11} & \bfW_{12} & \bfW_{13}\\
\bfW_{21} & \bfW_{22} & \bfW_{23}\\
\bfW_{31} & \bfW_{32} & \bfW_{33}
\end{bmatrix},
\end{equation*}
with \(\bfW_{11} \in \bbR^{d\times d}\), \(\bfW_{22} \in \bbR^{K\times K}\),
\(\bfW_{33} \in \bbR\), and \(z_j = 0\) for \(j \in [N]\),
\(z_{N+1} = 1\). Since \(\bfh_{N+1} = [\bfx_{N+1}; \vzero_K; 1]\), the score
expands as
\begin{align}\label{eq:score-blocks}
\bfh_j^\top \bfW \bfh_{N+1}
&= \bfx_j^\top\big(\bfW_{11}\bfx_{N+1} + \bfW_{13}\big) \\
&+ \bfu_{c_j}^\top\big(\bfW_{21}\bfx_{N+1} + \bfW_{23}\big)\nonumber\\
&+ z_j\big(\bfW_{31}\bfx_{N+1} + \bfW_{33}\big).\nonumber
\end{align}

We treat the blocks in three groups.

\emph{(i) The blocks \(\bfW_{12}, \bfW_{22}, \bfW_{32}\).} These do not occur
in \eqref{eq:score-blocks}: they multiply the second block of
\(\bfh_{N+1}\), which is \(\vzero_K\). Every score, hence every attention
weight, hence \(L\), is therefore constant in these blocks, and
\(\nabla_{\bfW_{12}} L = \nabla_{\bfW_{22}} L = \nabla_{\bfW_{32}} L = \vzero\)
identically --- not merely at the initialization. This is the multiclass form
of the observation that the label-to-output column does not influence the
prediction \cite[Section~2.2]{li2024onelayer}.

\emph{(ii) The blocks \(\bfW_{21}, \bfW_{13}, \bfW_{23}\).} By
\eqref{eq:score-blocks} these enter every score linearly, so by the chain
rule each gradient is an expectation of the form
\(\bbE[\sum_l \gamma_l \, \bfu_{c_l}^{\otimes m}\,(\cdot)]\) with weights
\(\gamma_l\) built from the \(q_j\) and the residual
\(\bell - \bfu_{c_{i^*}}\), and with \(m \le 3\); this is the counterpart of
\cite[Eqs.~(C.4)--(C.6)]{li2024onelayer}. Applying \eqref{eq:tower} and
\cref{lem:simplex-props} term by term: a monomial containing a code
\(\bfu_{c_l}\) whose label appears exactly once has conditional expectation
\(\bbE[\bfu_{c_l}\mid \bfx_{1:N+1}] = \vzero_K\) by \cref{lem:simplex-props}(i);
a monomial containing \(\bfu_{c_l}\bfu_{c_l}^\top\) reduces to the constant
\(\tfrac{K-1}{K}\) by \cref{lem:simplex-props}(iii), leaving a residual factor
that is again linear in a single code and so has zero mean; and a monomial
containing two distinct codes has conditional expectation
\(\vzero_{K\times K}\) by \cref{lem:simplex-props}(ii). Hence
\(\nabla_{\bfW_{21}} L = \nabla_{\bfW_{13}} L = \nabla_{\bfW_{23}} L = \vzero\).

\emph{(iii) The block \(\bfW_{31}\).} By \eqref{eq:score-blocks} it appears
only in the term \(z_j\bfW_{31}\bfx_{N+1}\), that is, only in the query's own
score, contributing \(\bfx_{N+1}^\top\) to the gradient. Under the induction
hypothesis the weights \(q_j\) are given by \eqref{eq:attn-scalar} and are
therefore even in \(\bfx_{1:N+1}\), while \(\bfx_{N+1}^\top\) is odd; since
the input law is invariant under \(\bfx_{1:N+1} \mapsto -\bfx_{1:N+1}\)
(\cref{ass:train}, property 3), the expectation vanishes. This argument uses
no label structure and is identical to the binary case.
The indicator block $\bfW_{33}$ has a partial derivative of 1 with respect to the query's self-attention score, guaranteeing a non-zero update during gradient descent: 
\begin{equation}
\begin{split}
\tfrac{1}{\eta}\big(\xi_2^{t+1} - \xi_2^t\big)
&= -\partial_{\xi_2} L(\bfW^t) \;=\; \nabla_{\bfW_{33}} L(\bfW^t) \\
&= \tfrac{K-1}{K}\,
\bbE\Big[q_{N+1}\Big(q_{i^*} - \textstyle\sum_{j=1}^{N} q_j^2\Big)\Big].
\end{split}
\end{equation}
Expanding
$\nabla_{\bfW_{11}} L$ and using \cref{lem:simplex-props} to discard the label
terms leaves \eqref{eq:grad-W11}, in which the inputs enter only through inner
products. By \cite[Lemma~7]{li2024onelayer} (rotational invariance of the uniform
distribution on $\calS^{d-1}$), such a gradient is a scalar multiple of the
identity, $\nabla_{\bfW_{11}} L(\bfW^t) = a_t \bfI_d$. With $\bfW_{11}^0 = \vzero$
this gives $\bfW_{11}^{t+1} = \xi_1^{t+1}\bfI_d$, completing the induction.
\end{proof}

\subsection{Two-Dimensional Reduction}

\begin{proof}[Proof of \cref{lem:2d}]
The claim is the diagonal form established in the induction of
\cref{lem:closed-form-grad}; it remains to be seen why the $K\times K$ label block
 does not contribute free parameters. At initialization $\bfW^0_{22} = \vzero_{K\times K}$
by \cref{ass:init}. By the structural-zero argument above, $\bfW_{22}$ multiplies
the query's (zero) label rows in every score, so $\nabla_{\bfW_{22}} L = \vzero$
at every step; a block that starts at zero and receives a zero gradient stays at
zero. Thus, although the ambient dimension grows from $(d+2)^2$ to $(d+K+1)^2$,
the trajectory explored by gradient descent is governed by the two scalars
$\xi_1^t$ and $\xi_2^t$, exactly as in the binary case.
\end{proof}

\subsection{The Scaling Identity}

We now prove \cref{lem:scaling}.

\begin{proof}[Proof of \cref{lem:scaling}]
By \eqref{eq:loss} and \eqref{eq:logits}, expanding the squared norm over
coordinates,
\begin{equation*}
L_{\mathrm{MC}}(\xi_1,\xi_2) = \tfrac12 \sum_{m=1}^{K} \bbE\!\left[
\Big(\sum_{j=1}^{N} q_j\,[\bfu_{c_j}]_m - [\bfu_{c_{i^*}}]_m\Big)^2\right],
\end{equation*}
with the input-only weights $q_j = q_j(\bfx,\bfW)$ of \eqref{eq:attn-scalar}. Expanding the square gives a
quadratic term, a cross term, and a constant; we sum each over $m$ using
\cref{lem:simplex-props}. The quadratic term is
$\sum_{j,j'} q_j q_{j'} \sum_{m=1}^{K} \bbE[[\bfu_{c_j}]_m [\bfu_{c_{j'}}]_m]
= \tfrac{K-1}{K}\sum_{j} q_j^2$, since
$\sum_m \bbE[[\bfu_{c_j}]_m^2] = \norm{\bfu_{c_j}}_2^2 = \tfrac{K-1}{K}$ by
\cref{lem:simplex-props}(iii) while distinct tokens are uncorrelated by
\cref{lem:simplex-props}(ii). The cross term
$\sum_{j} q_j \sum_m \bbE[[\bfu_{c_j}]_m [\bfu_{c_{i^*}}]_m]$ contributes only
through $j = i^*$, giving $\tfrac{K-1}{K}\, q_{i^*}$ in expectation. The constant
term is $\bbE[\norm{\bfu_{c_{i^*}}}_2^2] = \tfrac{K-1}{K}$. Each is exactly
$\tfrac{K-1}{K}$ times the corresponding term of the binary-analog loss, in which
$y \in \{\pm1\}$ gives $\bbE[y^2] = 1$ and $\bbE[y_j y_{j'}] = 0$ for
$j \neq j'$. Therefore $L_{\mathrm{MC}} = \tfrac{K-1}{K} L_{\mathrm{bin}}$, and
differentiating gives the gradient identity.
\end{proof}

\begin{lemma}[Nonconvexity]\label{lem:nonconvex}
The reduced loss $L_{\mathrm{MC}}(\xi_1,\xi_2)$ is nonconvex.
\end{lemma}

\begin{proof}
By \cref{lem:scaling}, $L_{\mathrm{MC}} = \tfrac{K-1}{K} L_{\mathrm{bin}}$ with
$\tfrac{K-1}{K} > 0$. The original binary reduced loss $L_{\mathrm{bin}}$ is
nonconvex \cite[Appendix~E]{li2024onelayer}, and a positive scalar multiple of a
nonconvex function is nonconvex; this is why $L_{\mathrm{MC}}$ is nonconvex, and
why the trajectory analysis below is needed.
\end{proof}

\subsection{Evolution of the Two Parameters}

Each bound below is the binary estimate of \cite{li2024onelayer} multiplied by
$\tfrac{K-1}{K}$ via \cref{lem:scaling}. The constants $c_1,\dots,c_4$, $c_1'$,
$c_2'$, $C_d$, and $a_{N,d} = (2N\sqrt{d})^{-2/(d-3)}$ are inherited unchanged
from the sphere geometry estimates \cite[Lemmas~17--20]{li2024onelayer}, which
have no labels.

\begin{lemma}[Increment bounds for $\xi_1$]\label{lem:xi1-incr}
For $\xi_1^t \ge 0$, there exist constants $c_1,c_2,c_3,c_4 > 0$ such that
\begin{align*}
\tfrac{d}{\eta}\big(\xi_1^{t+1} - \xi_1^t\big)
&\ge \tfrac{K-1}{K}\big(c_1\, e^{-6\xi_1^t} - c_2\, e^{2\xi_1^t - \xi_2^t}\big),\\
\tfrac{d}{\eta}\big(\xi_1^{t+1} - \xi_1^t\big)
&\le \tfrac{K-1}{K}\big(c_3\, e^{\poly(N,d)\,\xi_1^t}
- c_4\, e^{2(\xi_1^t - \xi_2^t)}\big).
\end{align*}
\end{lemma}

\begin{proof}
By \cref{lem:2d} the $\xi_1$ update is
$\xi_1^{t+1} - \xi_1^t = -\eta\,\partial_{\xi_1} L_{\mathrm{MC}}(\xi_1^t,\xi_2^t)$,
and from \cref{lem:scaling} we have
$\partial_{\xi_1} L_{\mathrm{MC}} = \tfrac{K-1}{K}\,\partial_{\xi_1} L_{\mathrm{bin}}$. Applying
\cite[Lemma~4]{li2024onelayer} to the binary gradient gives both inequalities.
\end{proof}

\begin{lemma}[Increment bounds for $\xi_2$]\label{lem:xi2-incr}
For $\xi_1^t \ge 0$, there exist constants $c_1', c_2' > 0$ such that
\begin{align*}
\tfrac{1}{\eta}\big(\xi_2^{t+1} - \xi_2^t\big)
&\ge \tfrac{K-1}{K}\, c_1'\, e^{-\poly(N,d)\,\xi_2^t},\\
\tfrac{1}{\eta}\big(\xi_2^{t+1} - \xi_2^t\big)
&\le \tfrac{K-1}{K}\, c_2'\, e^{-\poly(N,d)\,\xi_2^t}.
\end{align*} 
\end{lemma}

\begin{proof}
The $\xi_2$ update is
$\xi_2^{t+1} - \xi_2^t = -\eta\,\partial_{\xi_2} L_{\mathrm{MC}}$, with
$\partial_{\xi_2} L_{\mathrm{MC}} = \tfrac{K-1}{K}\,\partial_{\xi_2} L_{\mathrm{bin}}$
by \cref{lem:scaling}, and \cite[Lemma~5]{li2024onelayer} bounds the binary
gradient. In particular the parameter updates follow the recurrence relation
$b_{t+1} - b_t \ge C e^{-\alpha b_t}$ with $C,\alpha>0$, so
$\xi_2^t = \Omega\big(\poly(N,d)\,\log t\big)$ and $\xi_2^t \to \infty$.
Note that the factor $\tfrac{K-1}{K}$, like the step size $\eta$, is
absorbed into $C$ and therefore contributes only an additive
$\tfrac{1}{\alpha}\log(\alpha C)$ to $\xi_2^t$; since
$\tfrac{K-1}{K} \in [\tfrac12, 1)$, this offset is bounded uniformly in
$K$ and does not affect the rate.
\end{proof}

\begin{lemma}[Combined lower bound]\label{lem:combined-lower}
For all $t \ge 0$,
\begin{equation*}
\tfrac{1}{\eta}\big(d(\xi_1^{t+1} - \xi_1^t) + 2(\xi_2^{t+1} - \xi_2^t)\big)
\ge \tfrac{K-1}{K}\big(1 - \tfrac{1}{2N}\big) C_d\, e^{-6\xi_1^t}.
\end{equation*}
\end{lemma}

\begin{proof}
The left-hand side equals
$\tfrac{K-1}{K}\big(d\,\partial_{\xi_1} L_{\mathrm{bin}}
+ 2\,\partial_{\xi_2} L_{\mathrm{bin}}\big)$ by \cref{lem:scaling}, to which
\cite[Lemma~8]{li2024onelayer} applies. The binary estimate uses only
$\ip{\bfx_{i^*}}{\bfx_{N+1}} \ge \ip{\bfx_j}{\bfx_{N+1}}$ and the constant $C_d$
from \cite[Lemma~17]{li2024onelayer}, both unchanged by $K$.
\end{proof}

\begin{lemma}[Sharp upper bound for the $\xi_1$ increment]\label{lem:xi1-upper}
For $\xi_1^t \ge 0$ and $N \ge O(\sqrt{d}\log d)$,
\begin{equation*}
\tfrac{1}{\eta}\big(\xi_1^{t+1} - \xi_1^t\big)
\le \tfrac{K-1}{K}\left(\tfrac{2N}{d}\, e^{-\frac{4}{(N+1)^2}\xi_1^t}
- \tfrac{a_{N,d}}{dN^3 e}\, e^{2(\xi_1^t - \xi_2^t)}\right).
\end{equation*}
\end{lemma}

\begin{proof}
Apply \cref{lem:scaling} to \cite[Lemma~9]{li2024onelayer}; the constant
$a_{N,d}$ comes from the sphere tail bound \cite[Lemma~19]{li2024onelayer} and is
unchanged.
\end{proof}

\begin{lemma}[Refined lower bound for the $\xi_1$ increment]\label{lem:xi1-refined}
For $\xi_1^t \ge 0$,
\begin{equation*}
\tfrac{d}{\eta}\big(\xi_1^{t+1} - \xi_1^t\big)
\ge \tfrac{K-1}{K}\big[\big(1 - \tfrac{1}{2N}\big)C_d\, e^{-6\xi_1^t}
- 2\, e^{2\xi_1^t - 2\xi_2^t}\big].
\end{equation*}
\end{lemma}

\begin{proof}
Combine $d$ times \cref{lem:combined-lower} with $2$ times the upper bound of
\cref{lem:xi1-upper}; the common factor $\tfrac{K-1}{K}$ factors out.
\end{proof}

\begin{lemma}[Sharp lower bound for the $\xi_2$ increment]\label{lem:xi2-lower}
For $\xi_1^t \ge 0$,
\begin{equation*}
\tfrac{1}{\eta}\big(\xi_2^{t+1} - \xi_2^t\big)
\ge \tfrac{K-1}{K}\cdot\tfrac{1}{(N+1)^3 e}\, e^{2 a_{N,d}\xi_1^t - 2\xi_2^t}.
\end{equation*}
\end{lemma}

\begin{proof}
The binary proof bounds
$\tfrac{1}{\eta}(\xi_2^{t+1}-\xi_2^t)
\ge \bbE[q_{i^*}(\bfx,\bfW^t)\,q_{N+1}^2(\bfx,\bfW^t)]$; by \cref{lem:scaling}
the multiclass quantity acquires $\tfrac{K-1}{K}$. The remaining estimates --- the
concentration bound $1 - \ip{\bfx_{i^*}}{\bfx_{N+1}} \ge a_{N,d}$, which holds
with probability at least $\tfrac1e$ by \cite[Lemma~19]{li2024onelayer}, and $q_{i^*}^3 \ge (N+1)^{-3}$ --- involve only
geometry and are unchanged.
\end{proof}

\subsection{Ratio Bound and Growth Rates}

\begin{lemma}[Ratio bound]\label{lem:ratio}
If $\sigma = \xi_2^0 \ge 3\log\big(\tfrac{2N^4 d}{a_{N,d}}\big)$ and
$\xi_1^t \ge 0$ for all $t$, then $\xi_1^t \le \tfrac{7}{15}\xi_2^t$ for all
$t \ge 0$.
\end{lemma}

\begin{proof}
By \cref{lem:xi1-upper}, the increment $\xi_1^{t+1} - \xi_1^t$ is nonpositive once
\begin{equation*}
\tfrac{K-1}{K}\cdot\tfrac{2N}{d}\, e^{-\frac{4}{(N+1)^2}\xi_1^t}
< \tfrac{K-1}{K}\cdot\tfrac{a_{N,d}}{dN^3 e}\, e^{2(\xi_1^t - \xi_2^t)}.
\end{equation*}
The factor $\tfrac{K-1}{K} > 0$ cancels from both sides, so the resulting
condition on $(\xi_1^t, \xi_2^t)$ is exactly the threshold comparison of the
binary analysis: no label code enters it, and it involves only the two scalars
and the sphere constant $a_{N,d}$. Under the stated bound on $\sigma$,
\cite[Lemma~12]{li2024onelayer} shows that the condition is met whenever
$\tfrac{15}{7}\xi_1^t \ge \xi_2^t$, so $\xi_1^{t+1} - \xi_1^t \le 0$ there.
Hence $\xi_1^t$ cannot grow past $\tfrac{7}{15}\xi_2^t$; since $\xi_2^t$ does
not decrease by \cref{lem:xi2-lower}, the bound holds for all $t$. Only
$\xi_1^t \le c\,\xi_2^t$ for some constant $c < \tfrac12$ is used downstream,
in \cref{thm:convergence} and \cref{thm:shift}.
\end{proof}

\begin{lemma}[Growth rates]\label{lem:growth}
With $\sigma$ and $N$ as in \cref{thm:convergence},
$\ \xi_1^t, \xi_2^t = \Omega\big(\poly(N,d)\,\log t\big)$,
with $\xi_1^t \le \tfrac{7}{15}\xi_2^t$ for all $t \ge 0$.
\end{lemma}

\begin{proof}
For $\xi_2$, \cref{lem:xi2-incr,lem:xi2-lower} bound the increment between two
exponentially decaying rates,
$\tfrac{K-1}{K}\cdot\tfrac{\eta}{(N+1)^3 e}\, e^{-2\xi_2^t}
\le \xi_2^{t+1} - \xi_2^t \le \tfrac{K-1}{K}\,\eta\, e^{-\xi_2^t/15}$; comparison
with the corresponding continuous flow gives
$\xi_2^t = \Omega\big(\poly(N,d)\,\log t\big)$, the factors
$\tfrac{K-1}{K}$ and $\eta$ contributing only a bounded additive offset.
For $\xi_1$, \cref{lem:xi1-refined} gives $\xi_1^{t+1} - \xi_1^t \ge 0$ while
$8\xi_1^t \le 2\xi_2^t + \log\big(\tfrac{C_d(1-\frac{1}{2N})}{2}\big)$. This
condition holds through an initial phase because $\xi_2^0 = \sigma$ is large and
$\xi_2^t$ increases. During that phase $\xi_1^t$ grows logarithmically, and the
upper bound of \cref{lem:xi1-upper} caps it at $O(\poly(N,d)\log t)$. Hence
$\xi_1^t = \Omega\big(\poly(N,d)\,\log t\big)$. The bound
$\xi_1^t \le \tfrac{7}{15}\xi_2^t$ is \cref{lem:ratio}.
\end{proof} 

\begin{proof}[Proof of \cref{thm:convergence}]
By \cref{lem:scaling} applied to the binary loss bound
\cite[Lemma~14]{li2024onelayer}, the multiclass loss obeys
\begin{equation*}
\bbE\!\left[\norm{\bell_{\bfW^t} - \bfu_{c_{i^*}}}_2^2\right]
\le \tfrac{K-1}{K}\, O\!\left(\frac{N^3 k_d^2}{\xi_1^t}\right)
+ \tfrac{K-1}{K}\, e^{2\xi_1^t - \xi_2^t}.
\end{equation*}
Substituting the growth rates of \cref{lem:growth}, the first term is
$O\big(\tfrac{K-1}{K}\cdot \poly(N,d)/\log t\big)$, and since
$\xi_1^t \le \tfrac{7}{15}\xi_2^t$ by \cref{lem:ratio} gives
$2\xi_1^t - \xi_2^t \le -\tfrac{1}{15}\xi_2^t = O(-\log t)$, the second term
decays as $O(t^{-1/15})$. Both vanish as $t \to \infty$, which is the bound of
\cref{thm:convergence}.
\end{proof}


\section{Distribution Shift and Argmax Classification}\label{app:shift}

We prove \cref{thm:shift} and \cref{cor:argmax}. Throughout,
$\bell \coloneqq \bell_{\bfW^{T}}(\bfx_{N+1})$ and
$q_j \coloneqq q_j(\bfx, \bfW^{T})$. By the convergence analysis, $\bfW^{T}$ is
diagonal with $x$-block $\xi_1^{T} \bfI_d$ and indicator entry $-\xi_2^{T}$, so the
attention weights depend on the inputs only through inner products:

\begin{equation}\label{eq:qT}
q_j = \frac{\exp(\xi_1^{T} \ip{\bfx_j}{\bfx_{N+1}})}
{\sum_{l=1}^N \exp(\xi_1^{T} \ip{\bfx_l}{\bfx_{N+1}}) + \exp(\xi_1^{T} - \xi_2^{T})}
\end{equation}
for $j \in [N]$.

\begin{proof}[Proof of \cref{thm:shift}]
Since the query contributes $\vzero_K$ to \eqref{eq:logits}, we have
$\bell = \sum_{j=1}^N q_j \bfu_{c_j}$ and $\sum_{j=1}^N q_j = 1 - q_{N+1}$.
Writing the target as
$\bfu_{c_{i^*}} = (1 - q_{N+1})\,\bfu_{c_{i^*}} + q_{N+1}\,\bfu_{c_{i^*}}$ and
grouping,
\begin{align}
\bell - \bfu_{c_{i^*}}
&= \sum_{j=1}^N q_j \big(\bfu_{c_j} - \bfu_{c_{i^*}}\big)
- q_{N+1}\,\bfu_{c_{i^*}} \nonumber\\
&= \sum_{\substack{j \in [N] \\ c_j \neq c_{i^*}}}
q_j \big(\bfu_{c_j} - \bfu_{c_{i^*}}\big) - q_{N+1}\,\bfu_{c_{i^*}},
\label{eq:errsplit}
\end{align}
where the terms with $c_j = c_{i^*}$ cancel because $\bfu_{c_j} = \bfu_{c_{i^*}}$.
Taking norms in \eqref{eq:errsplit} and using
$\norm{\bfu_{c_j} - \bfu_{c_{i^*}}}_2 = \sqrt{2}$ and
$\norm{\bfu_{c_{i^*}}}_2 = \sqrt{(K-1)/K} \le 1$,
\begin{equation}\label{eq:normbound}
\norm{\bell - \bfu_{c_{i^*}}}_2
\le \sqrt{2}\!\!\sum_{\substack{j \in [N] \\ c_j \neq c_{i^*}}}\!\! q_j
+ q_{N+1}.
\end{equation}
On the event $A_\delta$ of \eqref{eq:margin}, every $j$ with $c_j \neq c_{i^*}$
satisfies $\norm{\bfx_j - \bfx_{N+1}}_2^2 \ge
\norm{\bfx_{i^*} - \bfx_{N+1}}_2^2 + \delta$. Since
$\norm{\bfx_a - \bfx_{N+1}}_2^2 = 2 - 2\ip{\bfx_a}{\bfx_{N+1}}$ on the sphere,
this is equivalent to
$\ip{\bfx_j}{\bfx_{N+1}} \le \ip{\bfx_{i^*}}{\bfx_{N+1}} - \tfrac{\delta}{2}$.
Hence, from \eqref{eq:qT} and $q_{i^*} \le 1$,
\begin{align}
\sum_{\substack{j \in [N] \\ c_j \neq c_{i^*}}} q_j
&\le \sum_{\substack{j \in [N] \\ c_j \neq c_{i^*}}} \frac{q_j}{q_{i^*}}
\nonumber\\
&= \sum_{\substack{j \in [N] \\ c_j \neq c_{i^*}}}
\exp\!\big(\xi_1^{T}(\ip{\bfx_j}{\bfx_{N+1}}
- \ip{\bfx_{i^*}}{\bfx_{N+1}})\big) \nonumber\\
&\le N \exp\!\big(-\tfrac{1}{2}\xi_1^{T} \delta\big).
\label{eq:wrongmass}
\end{align}
For the self-attention weight,
$q_{N+1}/q_{i^*} = \exp(\xi_1^{T}(1 - \ip{\bfx_{i^*}}{\bfx_{N+1}}) - \xi_2^{T}) \le
\exp(2\xi_1^{T} - \xi_2^{T})$, and by the convergence analysis
$\xi_1^{T} \le \tfrac{7}{15}\xi_2^{T}$, so $2\xi_1^{T} - \xi_2^{T} \le -\tfrac{1}{15}\xi_2^{T}$
and
\begin{equation}\label{eq:selfmass}
q_{N+1} \le \exp\!\big(-\tfrac{1}{15}\xi_2^{T}\big).
\end{equation}
Combining \eqref{eq:normbound}, \eqref{eq:wrongmass}, and \eqref{eq:selfmass},
\begin{equation}\label{eq:pointwise}
\norm{\bell - \bfu_{c_{i^*}}}_2
\le \sqrt{2}\,N \exp\!\big(-\tfrac{1}{2}\xi_1^{T} \delta\big)
+ \exp\!\big(-\tfrac{1}{15}\xi_2^{T}\big)
\end{equation}
uniformly on $A_\delta$. By \cref{lem:growth},
$\xi_1^{T}, \xi_2^{T} = \Omega(\poly(N,d)\log T)$, so both terms are
$O\big(N\,T^{-\poly(N,d)\delta}\big)$.

It remains to take expectations. Splitting on $A_\delta$ and its complement, and
using $\norm{\bell - \bfu_{c_{i^*}}}_2^2 \le (\sqrt2 + 1)^2 \cdot \tfrac{K-1}{K}$
as a hard uniform bound on $A_\delta^c$ (every $q_j \le 1$ and
$\norm{\bfu_c}_2^2 = \tfrac{K-1}{K}$),
\begin{align*}
&\bbE_{P^{\mathrm{test}}}\big[\norm{\bell - \bfu_{c_{i^*}}}_2^2\big]\\
&\ = \bbE_{P^{\mathrm{test}}}\big[\norm{\bell - \bfu_{c_{i^*}}}_2^2
\,\ind\{A_\delta\}\big]
+ \bbE_{P^{\mathrm{test}}}\big[\norm{\bell - \bfu_{c_{i^*}}}_2^2
\,\ind\{A_\delta^c\}\big]\\
&\ \le O\!\big(N^2 T^{-\poly(N,d)\delta}\big)
+ O\!\big(\tfrac{K-1}{K}\big)\, P^{\mathrm{test}}(A_\delta^c).
\end{align*}
Since $\tfrac{K-1}{K} \ge \tfrac12$ for every $K \ge 2$, the first term also
carries the factor $\tfrac{K-1}{K}$ after absorbing the $\sqrt2$ constants and a
further factor of two. Taking the infimum over
$\delta > 0$ gives the claim.
\end{proof}

\begin{lemma}[Argmax margin]\label{lem:argmax-margin}
For any $\bell \in \bbR^K$ and any class $c \in [K]$, if
$\norm{\bell - \bfu_c}_2 < \tfrac{1}{2}$ then $\argmax_{k\in[K]} [\bell]_k = c$.
\end{lemma}

\begin{proof}
Because $\norm{\cdot}_\infty \le \norm{\cdot}_2$, the hypothesis gives
$\big|[\bell]_k - [\bfu_c]_k\big| < \tfrac12$ for every coordinate $k$. The code
$\bfu_c$ has entry $[\bfu_c]_c = \tfrac{K-1}{K}$ and $[\bfu_c]_k = -\tfrac1K$ for
$k \neq c$, so
\begin{equation*}
[\bell]_c > \tfrac{K-1}{K} - \tfrac12, \qquad
[\bell]_k < -\tfrac1K + \tfrac12 \quad (k \neq c).
\end{equation*}
Subtracting, for every $k \neq c$,
\begin{align*}
[\bell]_c - [\bell]_k
&> \big(\tfrac{K-1}{K} - \tfrac12\big) - \big(-\tfrac1K + \tfrac12\big)\\
&= \tfrac{K-1}{K} + \tfrac1K - 1 = 0.
\end{align*}
Thus $[\bell]_c$ strictly exceeds every other coordinate, so the argmax is $c$.
\end{proof}

\begin{proof}[Proof of \cref{cor:argmax}]
Taking the contrapositive of \cref{lem:argmax-margin} with $c = c_{i^*}$,
\begin{equation*}
\big\{\hat{c}_{\bfW^{T}}(\bfx_{N+1}) \neq c_{i^*}\big\}
\subseteq \big\{\norm{\bell - \bfu_{c_{i^*}}}_2 \ge \tfrac12\big\}.
\end{equation*}
By Markov's inequality,
\begin{align*}
&P^{\mathrm{test}}\big(\hat{c}_{\bfW^{T}}(\bfx_{N+1}) \neq c_{i^*}\big)\\
&\quad\le P^{\mathrm{test}}\big(\norm{\bell - \bfu_{c_{i^*}}}_2^2
\ge \tfrac14\big)\\
&\quad\le 4\,\bbE_{P^{\mathrm{test}}}\big[\norm{\bell
- \bfu_{c_{i^*}}}_2^2\big].
\end{align*}
Applying \cref{thm:shift} and noting $\tfrac{K-1}{K} \le 1$ bounds the right-hand
side by $O\big(\inf_{\delta>0}\{N^2 T^{-\poly(N,d)\delta} +
P^{\mathrm{test}}(A_\delta^c)\}\big)$, which is the first claim.

For the second claim, suppose $P^{\mathrm{test}}(A_{\delta^*}) = 1$. Then
\eqref{eq:pointwise} holds almost surely, so
$\norm{\bell - \bfu_{c_{i^*}}}_2 \le \sqrt2\,N\exp(-\tfrac12\xi_1^{T}\delta^*) +
\exp(-\tfrac1{15}\xi_2^{T}) = O\big(N T^{-\poly(N,d)\delta^*}\big)$ almost surely.
By \cref{lem:argmax-margin}, the argmax is correct as soon as this bound falls
below $\tfrac12$, i.e.\ once
$\log T \ge O\big(\log(KN)/(\poly(N,d)\,\delta^*)\big)$.
\end{proof}

\section{Auxiliary Sphere-Geometry Results}\label{app:aux}

The convergence proof has several facts about the uniform distribution on
$\calS^{d-1}$. None of them involve the labels, so they are identical to their
counterparts in the binary analysis. We state them here for completeness and
refer to \cite{li2024onelayer} for the proofs, which are unchanged.

\begin{lemma}[Rotational invariance: {\cite[Lemmas~6--7]{li2024onelayer}}]
\label{lem:rot-inv}
If $\{\bfx_i\}_{i\in[N+1]}$ are i.i.d.\ uniform on $\calS^{d-1}$, then their joint
law is invariant under any orthogonal matrix $\bfU$. Consequently, if $\bfW$ is
diagonal with $x$-block $\xi_1 \bfI_d$, the gradient $\nabla_{\bfW_{11}} L(\bfW)$
is a scalar multiple of $\bfI_d$.
\end{lemma}

\begin{lemma}[Inner-product density: {\cite[Lemma~17]{li2024onelayer}}]
\label{lem:ip-density}
Let $\tau = \ip{\bfx}{\bfx'}$ for a fixed unit vector $\bfx$ and a uniform
$\bfx' \in \calS^{d-1}$. Then $\tau$ has density
$f_\tau(t) = k_d \,(1 - t^2)^{\frac{d-3}{2}}$ on $[-1,1]$, where
$k_d \coloneqq \dfrac{2\,\Gamma(\tfrac{d}{2})}{\sqrt{\pi}\,\Gamma(\tfrac{d-1}{2})}$.
\end{lemma}

\begin{lemma}[Expected nearest-neighbor alignment: {\cite[Lemma~18]{li2024onelayer}}]
\label{lem:nn-align}
If $N \ge O(\sqrt{d}\log d)$, then
$\bbE[\ip{\bfx_{i^*}}{\bfx_{N+1}}] \ge \dfrac{2}{(N+1)^2}$, where
$\bfx_{i^*} = \argmax_{i\in[N]} \ip{\bfx_i}{\bfx_{N+1}}$.
\end{lemma}

\begin{lemma}[Concentration of the nearest-neighbor alignment:
{\cite[Lemma~19]{li2024onelayer}}]\label{lem:nn-conc}
With $\bfx_{i^*}$ as above and $k_d$ as in \cref{lem:ip-density},
\begin{equation*}
\bbP\!\left(\ip{\bfx_{i^*}}{\bfx_{N+1}}
\le 1 - (2N k_d)^{-2/(d-3)}\right) \ge \tfrac{1}{e}.
\end{equation*}
Equivalently, with the constant $a_{N,d} = (2N\sqrt{d})^{-2/(d-3)}$ used in
\cref{lem:xi1-upper,lem:xi2-lower}, the gap satisfies
$1 - \ip{\bfx_{i^*}}{\bfx_{N+1}} \ge a_{N,d}$ with probability at least
$\tfrac1e$; that is, $a_{N,d}$ lower bounds the separation of the nearest
neighbor from the query with constant probability. Note $k_d = \Theta(\sqrt{d})$,
so the two forms agree up to constants.
\end{lemma}

\begin{lemma}[Order-statistic gap: {\cite[Lemma~20]{li2024onelayer}}]
\label{lem:gap}
Let $\ip{\bfx_{i^*}}{\bfx_{N+1}}$ and $\ip{\bfx_{(2)}}{\bfx_{N+1}}$ be the largest
and second-largest of $\{\ip{\bfx_i}{\bfx_{N+1}}\}_{i\in[N]}$. Then, for $\xi > 0$,
\begin{align*}
&\bbE\!\left[\exp\!\big(\xi(\ip{\bfx_{(2)}}{\bfx_{N+1}}
- \ip{\bfx_{i^*}}{\bfx_{N+1}})\big)\right]\\
&\qquad = O\!\left(\frac{N^2 k_d^2}{\xi}\right)
\quad\text{and}\quad
= \Omega\!\left(\frac{1}{\xi}\right).
\end{align*}
\end{lemma}


\section{Compute Environment}\label{app:hardware}
All experiments were conducted on a single workstation running Linux (kernel 6.8.0-58-generic) with glibc 2.39. The system was equipped with a 16-core (32-thread) CPU, 128 GB of RAM, and one NVIDIA RTX 5000 Ada Generation GPU.

\end{document}